\documentclass[12pt]{article}

\usepackage{makecell}
\usepackage{arydshln}
\usepackage{enumerate}
\usepackage{amsthm}
\usepackage{amssymb}
\usepackage{amsmath}
\usepackage{amsfonts}
\usepackage{cancel}
\usepackage{times}
\usepackage{bbm}
\usepackage{graphicx}
\usepackage[noadjust]{cite}
\usepackage{color, soul}
\usepackage{hyperref}
\usepackage{wasysym}
\hypersetup{
	colorlinks   = true,
	citecolor    = black,
	linkcolor    = blue
}
\usepackage{ragged2e}
\usepackage{subcaption}
\usepackage[paper=a4paper, top=2.5cm, bottom=2.5cm, left=2.5cm, right=2.5cm]{geometry}
\usepackage{lipsum}

\def\qed{\hfill $\vcenter{\hrule height .3mm
		\hbox {\vrule width .3mm height 2.1mm \kern 2mm \vrule width .3mm
			height 2.1mm} \hrule height .3mm}$ \bigskip}

\def \RR {\mathbb R}
\def \NN {\mathbb N}
\def \EE {\mathbb E}

\def \PP {\mathbb P}
\def \eps {\varepsilon}

\newtheorem{theorem}{Theorem}
\newtheorem{lemma}[theorem]{Lemma}

\newtheorem{claim}[theorem]{Claim}

\newtheorem{corollary}[theorem]{Corollary}
\theoremstyle{definition}
\newtheorem{definition}[theorem]{Definition}
\theoremstyle{remark}
\newtheorem{remark}[theorem]{Remark}

\newtheorem*{remark*}{Remark}

\long\def\symbolfootnotetext[#1]#2{\begingroup
	\def\thefootnote{\fnsymbol{footnote}}\footnotetext[#1]{#2}\endgroup}

\usepackage[utf8]{inputenc}

\title{Size and depth of monotone neural networks: interpolation and approximation }
\author{Dan Mikulincer\\
	Massachusetts Institute of Technology \\
	\and Daniel Reichman\\
	Worcester Polytechnic Institute}
\begin{document}
	\maketitle
	\begin{abstract}
		We study monotone neural networks with threshold gates where all the weights (other than the biases) are non-negative. We focus on the expressive power and efficiency of representation of such networks. Our first result establishes that every monotone function over $[0,1]^d$ can be approximated within arbitrarily small additive error by a depth-4 monotone network. When $d > 3$, we improve upon the previous best-known construction which has depth $d+1$. Our proof goes by solving the monotone interpolation problem for monotone datasets using a depth-4 monotone threshold network. In our second main result we compare size bounds between monotone and arbitrary neural networks with threshold gates. We find that there are monotone real functions that can be computed efficiently by networks with no restriction on the gates whereas monotone networks approximating these functions need exponential size in the dimension. 
	\end{abstract}

	\section{Introduction}
	The recent successes of neural networks are owed, at least in part, to their great approximation and interpolation capabilities. However, some prediction tasks require their predictors to possess specific properties. 
	This work focuses on monotonicity and studies the effect on overall expressive power when restricting attention to \emph{monotone neural networks}.
	
	Given $x,y \in \mathbb{R}^d$ we consider the partial ordering,
	\begin{equation} \label{eq:montone}
		x \geq y \iff \text{ for every } i =1,\dots d,\ [x]_i \geq [y]_i.
	\end{equation}
	Here, and throughout the paper, we use $[x]_i$ for the $i^{\mathrm{th}}$ coordinate of $x$.
	A function $f:[0,1]^d \rightarrow \mathbb{R}$ is \emph{monotone\footnote{As we will only be dealing with monotone increasing functions, we shall refer to monotone increasing functions as monotone.}} if for every two vectors $x,y \in [0,1]^d$, 
	\[
	x \geq y \implies f(x) \geq f(y).
	\]
	
	Monotone functions arise in several fields such as economics, operations research, statistics, computational complexity theory, healthcare, and engineering.
	For example, larger houses typically result in larger prices, and certain features are monotonically related to option pricing~\cite{dugas2000incorporating} and bond rating~\cite{daniels2010monotone}. As monotonicity constraints abound, there are specialized statistical methods aimed at fitting and modeling monotonic functions such as Isotonic Regression~\cite{brunk1972statistical,kalai2009isotron,kyng2015fast} as well as many other works related to monotone approximation~\cite{costantini1999local,hall2001nonparametric,willemans1996smoothing}. Neural networks are no exception: Several works are devoted to the study of approximating monotone functions using neural networks~\cite{sill1997monotonic,daniels2010monotone,sivaraman2020counterexample}. 

	When using a network to approximate a monotone function, one might try to ``force'' the network to be monotone. A natural way to achieve this is to consider only networks where every parameter (other than the biases) is non-negative\footnote{We restrict our attention to non-negative prediction problems: The domain of the function we seek to approximate does not contain vectors with a negative coordinate.}. Towards this aim, we introduce the following class of \emph{monotone networks}.
	
	Recall that the building blocks of a neural network are the single one-dimensional neurons $\sigma_{w,b}(x)=\sigma(\langle w, x\rangle+b)$ where $x \in \mathbb{R}^k$ is an input of the neuron, $w \in \mathbb{R}^k$ is a weight parametrizing the neuron, $\sigma:\mathbb{R} \rightarrow \mathbb{R}$ is the activation function and $b \in \mathbb{R}$ is a bias term. Two popular choices for activation functions, that we shall consider as well, are the ReLU activation function $\sigma(z)=\max(0,z)$ and the threshold activation function 
	$\sigma(z) = {\bf 1}(z \geq 0)$, which equals 1 if $z\geq 0$ and zero otherwise. We slightly abuse the term and say that a network is monotone if every single neuron is a monotone function. Since both ReLU and threshold are monotone, this requirement, of having every neuron a monotone function, translates to $w$ having all positive entries.
	
	Such a restriction on the weights can be seen as an inductive bias reflecting prior knowledge that the functions we wish to approximate are monotone. One advantage of having such a ``positivity bias'' is that it guarantees the monotonicity of the network. Ensuring that a machine learning model approximating a monotone function is indeed monotone is often desirable~\cite{gupta2016monotonic,milani2016fast,liu2020certified}. Current learning methods such as stochastic gradient descent and back-propagation for training networks are not guaranteed to return a network computing a monotone function even when the training data set is monotone. Furthermore, while there are methods for certifying that a given neural network implements a monotone function~\cite{liu2020certified,sivaraman2020counterexample}, the task of certifying monotonicity remains a non-trivial task. 
	
	Restricting the weights to be positive raises several questions on the behavior of monotone neural networks compared to their more general, unconstrained counterparts. For example, can monotone networks approximate arbitrary monotone functions within an arbitrarily small error?
	
	Continuing a line of work on monotone networks, \cite{sill1997monotonic, daniels2010monotone}, we further elucidate the above comparison and uncover some similarities (a universal approximation theorem) and some surprising differences (the interplay of depth and monotonicity) between monotone and arbitrary networks.
	We will mainly be interested in expressiveness, the ability to approximate monotone functions, and interpolate monotone data sets with monotone neural networks of constant depth.
	
	Some of our findings rely on results from Boolean circuit complexity. Boolean circuit complexity has been used before to address theoretical questions related to neural networks (e.g., ~\cite{hajnal1993threshold,chattopadhyay2021lower,shawe1992classes,parberry1994circuit}) and we believe that additional insights could be found by studying the intersection of circuit complexity and deep learning. 
	\subsection{Our contributions}
	\paragraph{On expressive power and interpolation}
	While it is well known that neural networks with ReLU activation are universal approximators (can approximate any continuous function on a bounded domain). Perhaps surprisingly, the same is not true for monotone networks and monotone functions. Namely, there are monotone functions that cannot be approximated within an arbitrary small additive error by a monotone network with ReLU gates regardless of the size and depth of the network. This fact was mentioned in~\cite{liu2020certified}: We provide a proof for completeness. 
	\begin{lemma}
		There exists a monotone function $f:[0,1] \to \mathbb{R}$ and a constant $c > 0$, such that for any monotone network $N$ with ReLU gates, there exists $x \in [0,1]$, such that
		\[
		|N(x) - f(x)| > c.
		\]
	\end{lemma}
	\begin{proof}
		It is known that a sum of convex functions $f_i, \sum\alpha_i f_i$ is convex provided that for every $i, \alpha_i \geq 0$. It is also known that the maximum of convex functions $g_i$, $\max_i \{g_i\}$ is a convex function. It follows from the definition of the ReLU gate (in particular, ReLU is a convex function) that a neural network with positive weights at all neurons is convex. As there are monotone functions that are not convex, the result follows. 
		
		For a concrete example, one may take the function $f(x) = \sqrt{x}$ for which the result holds with $c = \frac{1}{8}$.
	\end{proof}
	In light of the above, we shall hereafter only consider monotone networks with the threshold activation function and discuss, for now, the problem of interpolation. 
	Here a monotone data set is a set of $n$ labeled points $(x_i,y_i)_{i \in [n]} \in (\mathbb{R}^d \times \mathbb{R})^n$ with the property
	\[
	i\neq j \implies x_i \neq x_j \quad\text{and}\quad x_i \leq x_j \implies y_i \leq y_j.
	\]
	In the \emph{monotone interpolation problem} we seek to find a monotone network $N$ such that for every $i \in [n], N(x_i)=y_i$. 

	For general networks (no restriction on the weights) with threshold activation, it has been established, in the work of Baum \cite{baum1988capabilities}, that even with $2$ layers, for any labeled data set in $\mathbb{R}^d$, there exists an interpolating network.

	In the next lemma, we demonstrate another negative result, which shows an inherent loss of expressive power when transitioning to $2$-layered monotone threshold networks, provided that the dimension is at least two. We remark that when the input is real-valued (i.e., one-dimensional), an interpolating monotone network always exists. This fact is simple, and we omit proof: It follows similar ideas to those in ~\cite{daniels2010monotone}.
	\begin{lemma} \label{lem:impossible}
		Let $d \geq 2$. There exists a monotone data set $(x_i,y_i)_{i \in [n]} \in (\mathbb{R}^d \times \mathbb{R})^n$, such that any depth-$2$ monotone network $N$, with a threshold activation function must satisfy,
		\[
		N(x_i) \neq y_i,
		\]
		for some $i \in [n]$.
	\end{lemma}
	Given the above result, it may seem that, similarly to the case of monotone networks with ReLU activations, the class of monotone networks with threshold activations is too limited, in the sense that it cannot approximate any monotone function with a constant depth (allowing the depth to scale with the dimension was considered in \cite{daniels2010monotone}, see below). One reason for such a belief is that, for non-monotone networks, depth $2$ suffices to ensure universality. Any continuous function over a bounded domain can be approximated by a depth-$2$ network~\cite{barron1993universal,cybenko1989approximation,hornik1989multilayer} and this universality result holds for networks with threshold or ReLU as activation functions.  Our first main result supports the contrary to this belief. We establish a depth separation result for monotone threshold networks and show that monotone networks can interpolate arbitrary monotone data sets by slightly increasing the number of layers. Thereafter, a simple argument shows that monotone networks of bounded depth are universal approximators of monotone functions. As noted, this is in sharp contrast to general neural networks, where adding extra layers can affect the efficiency of the representation~\cite{eldan2016power}, but does not change the expressive power.
	
	\begin{theorem} \label{thm:construction}
		Let $(x_i,y_i)_{i \in [n]} \in (\mathbb{R}^d \times \mathbb{R})^n$ be a monotone data set. There exists a monotone threshold network $N$, with $4$ layers and $dn + 2n$ neurons such that,
		\[
		N(x_i) = y_i,
		\]
		for every $i \in [n]$.
		
		Moreover, if the set $(x_i)_{i \in [n]}$ is totally-ordered, in the sense that, for every $i, j \in [n]$, either $x_i \leq x_j$ or $x_j \geq x_i$, then one may take $N$ to have $3$ layers and $2n$ neurons.
	\end{theorem}
	We also complement Theorem $1$ with a lower bound that shows that the number of neurons we use is essentially tight, up to the dependence on the dimension. 
	\begin{lemma} \label{lem:lowerbound}
		There exists a monotone data set $(x_i,y_i)_{i \in [n]} \subset (\mathbb{R}^d \times \mathbb{R})^n$ such that, if $N$ is an interpolating monotone threshold network, the first layer of $N$ must contain $n$ units. Moreover, this lower bound holds when the set $(x_i)_{i\in [n]}$ is totally-ordered.
	\end{lemma}
	The lower bound of Lemma \ref{lem:lowerbound} demonstrates another important distinction between monotone and general neural networks. According to \cite{baum1988capabilities}, higher dimensions allow general networks, with $2$ layers, to be more compact. Since the number of parameters in the networks increases with the dimension, one can interpolate labeled data sets in general position with only $O\left(\frac{n}{d}\right)$ neurons. Moreover, for deeper networks, a recent line of work, initiated in \cite{vershynin2020memory}, shows that $O(\sqrt{n})$ neurons suffice. Lemma \ref{lem:lowerbound} shows that monotone networks cannot enjoy the same speed-up, either dimensional or from depth, in efficiency.

	Since we are dealing with monotone functions, our interpolation results immediately imply a universal approximation theorem for monotone networks of depth $4$.
	
	\begin{theorem} \label{thm:universalapprox}
		Let $f:[0,1]^d \to \mathbb{R}$ be a continuous monotone function and let $\varepsilon >0$. Then, there exists a monotone threshold network $N$, with $4$ layers, such that, for every $x \in[0,1]^d$,
		\[
		|N(x) - f(x)| \leq \varepsilon.
		\]
		If the function $f$ is $L$-Lipschitz, for some $L > 0$, one can take $N$ to have $(d+2)\left(\frac{L\sqrt{d}}{\eps}\right)^d$ neurons.
	\end{theorem}
	While it was previously proven, in \cite{daniels2010monotone}, that monotone networks with a threshold activation function can approximate any monotone function, the \emph{depth} in the approximating network given by \cite{daniels2010monotone} scales \emph{linearly} with the dimension. Our result is thus a significant improvement, whenever $d > 3$, which only requires constant depth. When looking at the size of the network, the linear depth construction, in \cite{daniels2010monotone} iteratively evaluates Riemann sums and builds a network that is piecewise constant on a grid. Hence, for $L$-Lipschitz functions, the size of the network would be comparable to the size guaranteed by Theorem \ref{thm:universalapprox}. 
	Whether one can achieve similar results to Theorems \ref{thm:construction} and \ref{thm:universalapprox} with only $3$ layers is an interesting question that we leave for future research.
	
	\paragraph{Efficiency when compared to general networks}
	We have shown that, with $4$ layers, monotone networks can serve as universal approximates. However, even if a monotone network can approximate a monotone function arbitrarily well, it might require a much larger size when compared to unconstrained networks. In this case, the cost of having a much larger network might outweigh the benefit of having a network that is guaranteed to be monotone.
	
	Our second main result shows that this can sometimes be the case. We show that using monotone networks to approximate, in the $\ell_{\infty}$-norm, a monotone function $h:[0,1]^d \rightarrow \mathbb{R}$ can lead to an exponential blow-up in the number of neurons. Namely, we give a construction of a smooth monotone function $h:[0,1]^d \rightarrow \mathbb{R}$ with a $\textrm{poly}(d)$ Lipschitz constant such that $h$ can be approximated within an additive error of $\eps>0$ by a general neural network with $\textrm{poly}(d)$ neurons. Yet any \emph{monotone} network approximating $h$ within error smaller than $\frac{1}{2}$ requires super-polynomial (exponential) size in $d$.
	
	\begin{theorem}\label{thm:monotoneeffic}
		There exists a monotone function $h:[0,1]^d \rightarrow \mathbb{R}$, such that:
		\begin{itemize}
			\item Any monotone threshold network $N$ which satisfies,
			\[
			|N(x) - h(x)| < \frac{1}{2}, \text{ for every } x \in [0,1]^d,
			\]
			must have $e^{d^{\alpha}}$ neurons, for some $\alpha > 0$. 
			\item For every fixed $\eps > 0$, there exists a general threshold network $N$, which has $\mathrm{poly}(d)$ neurons and such that,
			\[
			|N(x) - h(x)| <  \eps, \text{ for every } x \in [0,1]^d.
			\]
		\end{itemize}
	\end{theorem}
	
	The function $h$ we use is the harmonic extension of a graph invariant function, introduced by \'Eva Tardos in \cite{tardos1988gap}. The Tardos function and its properties build upon the seminal works of Razborov \cite{razborov1985lower}, and Alon and Boppana \cite{alon1987monotone}. The mentioned works constitute a highly influential line of work, about the limitation of computation using monotone circuits. This line of work culminated in the introduction of the monotone Tardos function, which is computable by a De-Morgan Boolean circuit of polynomial size\footnote{The size of a circuit is the number of gates it has.}, yet requires an exponential number of gates in any representation by a monotone De Morgan circuit without negations. 
	
	To make use of these foundational results in circuit complexity theory, we shall need to make a connection with neural networks. As part of the proof of Theorem \ref{thm:monotoneeffic}, we introduce two reductions going from, and to, neural networks and Boolean circuits. Our reductions are quite general and may be of independent interest.
	\section{Related work}
	We are unaware of previous work studying the interpolation problem for monotone data sets using monotone networks. There is extensive research regarding the size and depth needed for general data sets and networks to achieve interpolation~\cite{zhang2021understanding,bubeck2020network,daniely2020neural,vershynin2020memory} starting with the seminal work of Baum~\cite{baum1988capabilities}. Known constructions of neural networks achieving interpolations are non-monotone: they may result in negative parameters even for monotone data sets. 
	
	Several works have studied approximating monotone (real) functions over a bounded domain using a monotone network. Sill~\cite{sill1997monotonic} provides a construction of a monotone network (all parameters are non-negative) with depth $3$ where the first layer consists of linear units divided into groups, the second layer consists of $\max$ gates where each group of linear units of the first layer is fed to a different gate and a final gate computing the minimum of all outputs from the second layer. It is proven in ~\cite{sill1997monotonic} that this class of networks can approximate every monotone function over $[0,1]^d$. We remark that this is very different from the present work's setting. First, using both $\min$ and $\max$ gates in the same architecture with positive parameters does not fall into the modern paradigm of an activation function.
	Moreover, we are not aware of works prior to Theorem \ref{thm:universalapprox} that show how to implement or approximate $\min$ and $\max$ gates, with arbitrary fan-ins, using constant depth \emph{monotone networks}\footnote{There are constructions of depth-3 threshold circuits with discrete inputs that are given $m$ numbers each represented by $n$ bits and output the (binary representation of the) maximum of these numbers~\cite{siu1991power}. This setting is different from ours, where the inputs are real numbers.}. Finally,~\cite{sill1997monotonic} focuses on approximating monotone functions and does not consider the monotone interpolation problem studied here.   
	
	Later, the problem of approximating arbitrary monotone functions with networks having non-negative parameters using more standard activation functions such as thresholds or sigmoids has been studied in~\cite{daniels2010monotone}. In particular~\cite{daniels2010monotone} gives a recursive construction showing how to approximate in the $\ell_{\infty}$ norm an arbitrary monotone function using a network of depth $d+1$ ($d$-hidden layers) with threshold units and non-negative parameters. In addition~\cite{daniels2010monotone} provides a construction of a monotone function $g:[0,1]^2 \rightarrow \mathbb{R}$ that cannot be approximated in the $\ell_{\infty}$ norm with an error smaller than $1/8$ by a network of depth $2$ with sigmoid activation and non-negative parameters, regardless of the number of neurons in the network. Our Lemma \ref{lem:impossible} concerns networks with threshold gates and applies to arbitrary dimensions larger than $1$. It can also be extended to provide monotone functions that monotone threshold networks cannot approximate with only two layers.
	
	Lower bounds for monotone models of computation have been proven for a variety of models~\cite{de2022guest}, including monotone De Morgan\footnote{Circuits with AND as well as OR gates without negations.} circuits~\cite{razborov1985lower,harnik2000higher,raz1997separation,garg2018monotone}, monotone arithmetic circuits and computations~\cite{chattopadhyay2021lower,yehudayoff2019separating,jerrum1982some}, which correspond to polynomials with non-negative coefficients, and circuits with monotone real gates~\cite{pudlak1997lower,hrubevs2018note} whose inputs and outputs are \emph{Boolean}. 
 Finally, our model of computation of neural networks with threshold gates differs from arithmetic circuits~\cite{shpilka2010arithmetic} which use gates that compute polynomial functions.
	
	To achieve our separation result, we begin with a Boolean function $m$, which requires super polynomial-size to compute with Boolean circuits with monotone threshold gates but can be computed efficiently with arbitrary threshold circuits: the existence of $m$ follows from~\cite{razborov1985lower}.  
	Thereafter, we show how to smoothly extend $m$ to have domain $[0,1]^d$ while preserving monotonicity. Proving lower bounds for neural networks with a continuous domain by extending a Boolean function $f$ for which lower bounds are known to a function $f'$ whose domain is $[0,1]^d$ has been done before~\cite{vardi2021size,yarotsky2017error}. However, the extension method in these works does not yield a function that is monotone. Therefore, we use a different method based on the multi-linear extension. 
	
	\section{Preliminaries and notation}
	We work on $\mathbb{R}^d$, with the Euclidean inner product $\langle \cdot, \cdot \rangle$. For $k \in \mathbb{N}$, we denote $[k] = \{1,2,\dots, k\}$ and use $\{e_i\}_{i \in [d]}$ for standard unit vectors in $\mathbb{R}^d$. That is, for $i \in [d]$, 
	\[
	e_i = (\underbrace{0,\dots,0}_{i-1 \text{ times}},1,\underbrace{0,\dots,0}_{d-i \text{ times}} ).
	\]
	For $x \in \mathbb{R}^d$ and $i \in [d]$, we write $[x]_i := \langle x, e_i \rangle$, the $i^{\mathrm{th}}$ coordinate of $x$.
	
	With a slight abuse of notation, when the dimension of the input changes to, say, $\mathbb{R}^k$, we will also use $\{e_i\}_{i\in [k]}$ to stand for standard unit vectors in $\mathbb{R}^k$. To avoid confusion, we will always make sure to make the dimension explicit.
	
	A neural network of depth $L$ is a function $N:\mathbb{R}^d \to \mathbb{R}$, which can be written as a composition,
	\[
	N(x) = N_L(N_{L-1}(\dots N_2(N_1(x))\dots),
	\]
	where for $\ell \in [L]$, $N_\ell:\mathbb{R}^{d_\ell} \to \mathbb{R}^{d_{\ell+1}}$ is a layer. We set $d_1 = d$, $d_{L+1} = 1$ and term $d_{\ell+1}$ as the width of layer $\ell$.
	Each layer is composed of single neurons in the following way: For $i \in [d_{\ell+1}]$, $[N_\ell(x)]_i = \sigma(\langle w^\ell_i, x\rangle + b^\ell_i) $
	where $\sigma:\mathbb{R} \to \mathbb{R}$ is the activation function, $w^\ell_i \in \mathbb{R}^{d_\ell}$ is the weight vector, and $b^\ell_i \in \mathbb{R}$ is the bias.
	The only exception is the last layer which is an affine functional of the previous layers, 
	\[
	N_L(x) = \langle w^L,x \rangle + b^L,
	\]
	for a weight vector $w^L \in \mathbb{R}^{d_L}$ and bias $b^L \in \mathbb{R}.$
	
	Suppose that the activation function is monotone. We say that a network $N$ is monotone, if, for every $\ell \in [L]$ and $i \in [d_{\ell + 1}]$, the weights vector $w_i^\ell$ has all positive coordinates. In other words,
		\[
		[w_i^{\ell}]_j \geq 0, \text{ for every } j \in [d_\ell].
		\]
	
	\section{A counter-example to expressibility}
	In this section, for every $d \geq 2$, we will construct a monotone data set such that any $2$-layered monotone threshold network cannot interpolate, thus proving Lemma \ref{lem:impossible}. Before proving the Lemma it will be instructive to consider the case $d = 2$. Lemma \ref{lem:impossible} will follow as a generalization of this case\footnote{	The 2D example here as well its generalization for higher dimensions can be easily adapted to give an example of a monotone function that cannot be approximated in $\ell_2$ by a depth-two monotone threshold network.}.  
	
	Recall that $\sigma(x) = {\bf 1}(x\geq 0)$ is the threshold function and consider the monotone set, 
		\begin{align*}
			x_1 &= (2,0),\ y_1 = 0\\
			x_2 &= (0,2),\ y_2 = 0\\
			x_3 &= (1,1),\ y_3 = 1.
		\end{align*}
	Assume towards a contradiction that there exists an interpolating monotone network 
	$N(x) := \sum\limits_{m =1}^ra_m\sigma\left(\langle x, w_m\rangle - b_m\right).$
	Set,
	\begin{align*}
	I &=\{m \in [r]| \langle x_3, w_m\rangle \geq b_m\}\\
	&= \{ m \in [r]| [w_m]_1 + [w_m]_2  \geq b_m\}.
	\end{align*}
	The set $I$ is the set of all neurons that are active on $x_3$, and, since $N(x_3) = 1$, $I$ is non-empty. We also define, 
	\begin{align*}
		I_1 &= \{m \in I| [w_m]_1 \geq [w_m]_2\}\\ 
		I_2 &= \{m \in I| [w_m]_2 \geq [w_m]_1\}.
	\end{align*}
	It is clear that $I_1 \cup I_2 = I.$ 
	Observe that for $m \in I_1$, by monotonicity, we have 
		\begin{equation*} 
			\langle x_1, w_m\rangle = 2[w_m]_1 \geq [w_m]_1 + [w_m]_2 = \langle x_3, w_m \rangle.
		\end{equation*}
	Since the same also holds for $m \in I_2$ and $x_2$, we get
	\begin{align} \label{eq:compare}
		N(x_1) + N(x_2) &\geq \sum\limits_ {m \in I_1}a_m \sigma\left(\langle x_1, w_m\rangle - b_m\right)\nonumber\\ &\ \ \ + 
		\sum\limits_ {m \in I_2}a_m \sigma\left(\langle x_2, w_m\rangle - b_m\right)\nonumber\\
		&\geq \sum\limits_ {m \in I_1} a_m \sigma\left(\langle x_3, w_m\rangle - b_m\right)\nonumber \\&\ \ \ + \sum\limits_ {m \in I_2}a_m \sigma\left(\langle x_3, w_m\rangle - b_m\right)\nonumber\\
		&\geq \sum\limits_ {m \in I} a_m \sigma\left(\langle x_3, w_m\rangle - b_m\right)\nonumber\\
		&= N(x_1) = 1.
	\end{align}
	Hence, either $N(x_1) \geq \frac{1}{2}$ or $N(x_2) \geq \frac{1}{2}$.\\
	
	With the example of $d = 2$ in mind, we now prove Lemma \ref{lem:impossible}.
	\begin{proof}[Proof of Lemma \ref{lem:impossible}]
		Consider the following monotone set of $d+1$ data points in $\mathbb{R}^d$ with $d \geq 2$. For $i \in [d], x_i = d\cdot e_i$, is a vector whose $i^{\mathrm{th}}$ coordinate is $d$ and all other coordinates are $0$, and set $y_i = 0$. We further set $x_{d+1}=(1,\dots,1)$ (the all $1$'s vector) with $y_{d+1} = 1$. 
		
		Suppose towards a contradiction there is a monotone depth-$2$ threshold network 
		\[
		N(x)=\sum_{m=1}^{r}a_{m}\sigma(\langle x,w_{m}\rangle-b_{m}),
		\]
		with $N(x_{d+1})=1$ and for every $i\in[d], N(x_i) = 0$.
		We prove the result while assuming that the bias of the output layer is $0$. Since the bias just adds a constant to every output it is straightforward to take it into account.\\
		
		Denote,
		\[
		I := \left\{m \in [r]| \sum_{i=1}^k[w_{m}]_i\geq b_{m}\right\}.
		\]
		Since $N(x_{d+1})=1$, we have that $I$ is non-empty. For $j \in [d]$, let 
		\[
		I_j: =\left\{m \in I|[w_m]_j=\max\{[w_m]_1,\dots,[w_{m}]_d\}\right\}.
		\]
		Clearly $I=\bigcup\limits_{j=1}^d I_j$ and we can assume, with no loss of generality, that this is a disjoint union. 
		Now, by following the exact same logic as in \eqref{eq:compare},
		\begin{align*}
			\sum_{i=1}^dN(x_i) &\geq \sum_{m \in I_1}a_{m}\sigma(\langle x_1,w_{m}\rangle-b_{m})+ \ldots \\
			&\ \ \ +\sum_{m \in I_d}a_{m}\sigma(\langle x_d,w_{m}\rangle-b_{m})\\
			& \geq \sum\limits_{m \in I}  a_{m}\sigma(\langle x_{d+1},w_{m}\rangle-b_{m}) =1
		\end{align*}
		Therefore there exists $j \in [k]$ with $N(x_j) \geq \frac{1}{d} > 0$, which is a contradiction.
	\end{proof}
	
	\section{Four layers suffice with threshold activation} \label{sec:construction}
	Let $(x_i,y_i)_{i=1}^n\in (\mathbb{R}^d\times \mathbb{R})^n$ be a monotone data set, and assume, with no loss of generality, 
	\begin{equation} \label{eq:labelorder}
		0\leq y_1 \leq y_2 \leq\dots \leq y_n.
	\end{equation}
	If, for some $i, i'\in [n]$ with $i \neq i'$ we have $y_i = y_{i'}$, and $x_i\leq x_{i'}$, we will assume $i < i'$. Other ties are resolved arbitrarily. Note that the assumption that the $y_i$'s are positive holds without loss of generality, as one can always add a constant to the output of the network to handle negative labels.
	
	This section is dedicated to the proof of Theorem \ref{thm:construction}, and we will show that one can interpolate the above set using a monotone network with $3$ hidden layers. The first hidden layer is of width $dn$ and the second and third of width $n$.
	
	Throughout we shall use $\sigma(t) = {\bf1}(t\geq 0)$, for the threshold function. For $\ell \in \{1,2,3\}$ we will also write $(w_i^\ell, b_i^\ell)$ for the weights of the $i^{\mathrm{th}}$ neuron in level $\ell$. We shall also use the shorthand, 
	\[
	\sigma_i^\ell(x) = \sigma(\langle x, w_i^\ell\rangle-b_i^\ell).
	\]
	
	We first describe the first two layers. The second layer serves as a monotone embedding into $\RR^n$. We emphasize this fact by denoting the second layer as $E:\mathbb{R}^d \to \mathbb{R}^n$, with $i^{\mathrm{th}}$ coordinate given by,  
	\[
	[E(x)]_i = \sigma_i^2(N_1(x)),
	\]
	where $[N_1(x)]_j = \sigma_j^1(x)$, for $j = 1,\dots, nd$, are the outputs of the first layer.
	\paragraph{First hidden layer} The first hidden layer has $dn$ units. 
	Let $e_i$ be the $i^{\mathrm{th}}$ standard basis vector in $\mathbb{R}^d$ and, for $j = 1,\dots, dn$ define 
	\[
	\sigma^1_j(x) := \sigma\left(\langle x, e_{(j\text{ mod } d) +1}\rangle -  \langle x_{\lceil \frac{j}{d}\rceil}, e_{(j \text{ mod } d) +1}\rangle\right).
	\]
	In other words, $w^1_j = e_{(j\text{ mod } d) +1}$ and $b^1_j = \langle x_{\lceil \frac{j}{d}\rceil}, e_{(j \text{ mod } d) +1}\rangle$ (the addition of $1$ offsets the fact that mod operations can result in $0$). To get a feeling of what the layer does, suppose that $j \equiv r \mod d$, then unit $j$ is activated on input $x$ iff the $(r+1)^{\mathrm{th}}$ entry of $x$ is at least the $(r+1)^{\mathrm{th}}$ entry of $x_{\lceil \frac{j}{d}\rceil}.$

	\paragraph{Second hidden layer} The second layer has $n$ units. For $j = 1,\dots, nd$, with a slight abuse of notation we now use $e_j$ for the $j^{\mathrm{th}}$ standard basis vector in $\mathbb{R}^{nd}$ and define unit $i =1,\dots,n$, $\sigma_i^2:\RR^{nd}\to \RR$, by
	\[
	\sigma_i^2(y) = \sigma\left(\sum_{r = 1}^d \langle y, e_{d(i-1) + r}\rangle -d \right).
	\]
	Explicitly, $w^2_i = \sum_{r = 1}^d e_{d(i-1) + r}$ and $b^2_i = d$.
	With this construction in hand, the following is the main property of the first two layers.
	\begin{lemma} \label{lem:montoneembed}
		Let $i = 1,\dots,n$. Then, $[E(x)]_i = 1$ if and only if $x \geq x_i$. Otherwise, $[E(x)]_i = 0$.
	\end{lemma}
	\begin{proof}
		By construction, we have $[E(x)]_i = 1$ if and only if $\sum_{r=1}^d \sigma^1_{d(i-1) + r}(x) \geq d$.
		For each $r \in [d]$, $\sigma^1_{d(i-1)+r}(x) \in \{0,1\}$. Thus, $[E(x)]_i = 1$ if and only if, for every $r \in [d]$, $\sigma^1_{d(i-1)+r}(x) = 1$. But $\sigma^1_{d(i-1)+r}(x) = 1$ is equivalent to $[x]_r \geq [x_i]_r$. Since this must hold for every $r \in [d]$, we conclude $x \geq x_i$.
	\end{proof}
	The following corollary is now immediate.
	\begin{corollary} \label{cor:embedprop}
		Fix $j \in [n]$ and let $i \in [n]$. 
		\begin{itemize}
			\item If $j < i$, then $[E(x_j)]_i = 0$.
			\item If $j \geq i$, then there exists $i' \geq i$ such that $[E(x_j)]_{i'} = 1$.
		\end{itemize}
	\end{corollary}
	\begin{proof}
		For the first item, if $j < i$, by the ordering of the labels \eqref{eq:labelorder}, we know that $x_j \ngeq x_i$. By Lemma \ref{lem:montoneembed},  $[E(x_j)]_i = 0$.
		
		For the second item, by construction, $[E(x_j)]_j = 1$. Since $j \geq i$, the claim concludes.
	\end{proof}
	\paragraph{The third hidden layer}
	The third layer contains $n$ units with weights given by
	\[
	w^3_i = \sum_{r=i}^n e_r \text{ and } b^3_i = 1.
	\]
	Thus,
	\begin{equation} \label{eq:thirdlay}
		\sigma_i^3(E(x)) = \sigma\left(\sum_{r=i}^n [E(x)]_r - 1\right).
	\end{equation}
	\begin{lemma} \label{lem:thirdlayerprop}
		Fix $j \in [n]$ and let $i \in [n]$. $\sigma_i^3(E(x_j)) =  1$ if $j \geq i$ and $\sigma_i^3(E(x_j)) =  0$ otherwise.
	\end{lemma}
	\begin{proof}
		By Corollary \ref{cor:embedprop}, $[E(x_j)]_r = 0$, for every $r > j$. In particular, if $i > j$, by \eqref{eq:thirdlay}, we get
			\[
			\sigma_i^3(E(x_j)) = \sigma(-1) = 0.
			\]
		On the other hand, if $j\geq i$, then by Corollary \ref{cor:embedprop}, there exists $i' \geq i$ such that $[E(x_j)]_{i'} = 1$, and
			\[
			\sigma_i^3(E(x_j)) = \sigma([E(x_j)]_{i'}-1) = 1.
			\]
	\end{proof}

	\paragraph{The final layer}
	The fourth and final layer is a linear functional of the output of the third layer. Formally, for $x \in \mathbb{R}^d$, the output of the network is, $N(x) = \sum\limits_{i=1}^n[w^4]_i\sigma^3_i(E(x))$, for some weights vector $w^4 \in \mathbb{R}$. To complete the construction we now define the entries of $w^4$, as $[w^4]_i = y_i -y_{i-1}$ with $y_{0} = 0$.
	We are now ready to prove Theorem \ref{thm:construction}.
	\begin{proof}[Proof of Theorem \ref{thm:construction}]
		Consider the function $N(x) =  \sum\limits_{i=1}^n[w^4]_i\sigma^3_i(E(x))$ described above. Clearly, it is a network with $3$ hidden layers. To see that it is monotone, observe that for $\ell \in \{1,2,3\}$ each $w_i^\ell$ is a sum of standard basis vectors, and thus has non-negative entries.
		The weight vector $w^4$ also has non-negative entries, since, by assumption, $y_i \geq y_{i-1}$.
		
		We now show that $N$ interpolates the data set $(x_j, y_j)_{j=1}^n$. Indeed, fix $j \in [n]$.
		By Lemma \ref{lem:thirdlayerprop}, we have
			\begin{align*}
				N(x_j) &= \sum\limits_{i=1}^n[w^4]_i\sigma^3_i(E(x_j))=\sum\limits_{i=1}^j[w^4]_i\\ &= \sum\limits_{i=1}^j(y_i - y_{i-1}) = y_j - y_{0} = y_j.
			\end{align*}
		The proof is complete, for the general case.\\
		
		To handle the case of totally-ordered $(x_i)_{i \in [n]}$, we slightly alter the construction of the first two layers, and compress them into a single layer satisfying Lemma \ref{lem:montoneembed}, and hence Lemma \ref{lem:thirdlayerprop}.
		
		The total-order of $(x_i)_{i\in [n]}$ implies the following: For every $i \in [n]$, there exists $r(i) \in [d]$, such that for any $j \in [n]$, $[x_i]_{r(i)} < [x_j]_{r(i)}$ if and only if $i < j$. In words, for every point in the set, there exists a coordinate that separates it from all the smaller points.
		We thus define $w_i^1 = e_{r(i)}$ and $b_i = 1$. 
		
		From the above, it is clear that 
			\[
			\sigma_i^1(x_j) = \sigma([x_j]_{r(i)} -1)= \begin{cases} 1& \text{if } i\leq j\\
				0& \text{if } i > j\end{cases}.
			\]
		\newline As in the general case, we define $E(x): \mathbb{R}^d \to \mathbb{R}^n$ by $[E(x)]_i = \sigma_i^1(x)$, and note that Lemma \ref{lem:montoneembed} holds for this construction of $E$ as well. The next two layers are constructed exactly like in the general case and the same proof holds.
		\end{proof}
	\subsection{Interpolating monotone networks are wide}
	The network we've constructed in Theorem \ref{thm:monotone} uses $(d+2)n$ neurons to interpolate a monotone data set of size $n$. One may wonder whether this can be improved. It turns out, that up to the dependence on $d$, the size of our network is essentially optimal. We now prove Lemma \ref{lem:lowerbound}.
	\begin{proof}[Proof of Lemma \ref{lem:lowerbound}]
		Let $(x_i,y_i)_{i \in [n]} \in \left(\mathbb{R}^d\times \mathbb{R}\right)^n$ be a monotone data set, such that for every $1 \leq i \leq n-1$, $x_i \leq x_{i+1}$ and $y_i \neq y_{i+1}$.
		Suppose that the first layer of $N$ has $k$ units and denote it by $N_1:\mathbb{R}^d \to \mathbb{R}^k$. Now, let $1\leq i < j \leq n$. 
		Since, every unit is monotone, and since $x_i \leq x_j$, we have for $\ell \in [k]$ the implication 
		$$[N_1(x_i)]_\ell \neq 0 \implies [N_1(x_i)]_\ell = [N_1(x_j)]_\ell.$$
		Denote $I_i = \{ \ell \in [k]| [N_1(x_i)]_\ell \neq 0\}$. The above shows the following chain condition,
		$$I_1 \subset I_2 \subset \dots \subset I_n.$$
		Suppose that $k < n$, then as $\{I_i\}_{i=1}^n$ is an ascending chain of subsets in $[k]$, necessarily there exists an $i < n$, such that $I_i = I_{i+1}$, which implies $N_1(x_i) = N_1(x_{i+1})$. We conclude that $N(x_i) = N(x_{i+1})$, which cannot happen as $y_i \neq y_{i+1}$. Hence, necessarily $k > n$.
	\end{proof}

	\begin{remark}
		It is known that general threshold networks (no restriction on the weights) can memorize $n$ points in $\RR^d$ using $O(\sqrt{n}+f(\delta))$ neurons where $f$ is a function depending on the minimal distance between any two of the data points~\cite{vershynin2020memory,rajput2021exponential}. Hence the lower bound in Lemma~\ref{lem:lowerbound} shows that there exist monotone data sets such that interpolation with monotone networks entails a quadratic blowup in the number of neurons of the network. In particular, observe that the lower bound in Lemma~\ref{lem:lowerbound} holds even if the data set is well-separated.
	\end{remark}
	\subsection{Universal approximation} \label{sec:universal}
	We now show that our interpolation scheme can be used to approximate any continuous monotone function, which is Theorem \ref{thm:universalapprox}.
	
	\begin{proof}[Proof of Theorem \ref{thm:universalapprox}]
		Since $f$ is continuous and $[0,1]^d$ is compact, $f$ is uniformly continuous.
		Hence, there exists $\delta > 0$, such that
		\[
		\|x - y\| \leq \delta \implies |f(x) - f(y)| \leq \varepsilon, \text{ for all } x,y\in[0,1]^d.
		\]
		Set $\delta_d = \frac{\delta}{\sqrt{d}}$ and consider the grid,
		\[
		G_\delta = (\delta_d\mathbb{Z})^d \cap [0,1]^d.
		\]
		I.e. $G_\delta$ is a uniform grid of points with spaces of width $\delta_d$.
		Define now a monotone data set $(x_j, f(x_j))_{x_j \in G_\delta}$. By Theorem \ref{thm:construction} there exists a network $N$, such that $N(x_j) = f(x_j)$ for every $x_j \in G_\delta$. 
		We claim that this is the approximating network. Let $x \in [0,1]^d$ and let $x_-$ (resp. $x_+$) be the closest point to $x$ in $G_\delta$ such that $x\geq x_-$ (resp. $x\leq x_+$).
		Observe that, $x_-$ and $x_+$ are vertices of a sub-cube in $[0,1]^d$ of side length $\delta_d$. Thus, $\|x_- - x_+\| \leq \sqrt{d}\delta_d=\delta.$
		Now, since both $N$ and $f$ are monotone,
    	\begin{align*}
    	    |N(x) - f(x)| 
    	    &\leq \max\left(|N(x_+) - f(x_-)|,|f(x_+) - N(x_-)| \right)\\
    	    &= |f(x_+) - f(x_-)| \leq \varepsilon.
    	\end{align*}

		Assume now that $f$ is $L$-Lipschitz and Let us now estimate the size of the network. According to Theorem \ref{thm:construction}, the network has $(d+2)|G_\delta|$ neurons. Since $f$ is $L$-Lipschitz, one can control the uniform continuity parameter and take $\delta = \frac{\eps}{L}$. Hence,
		\[
		|G_\delta| = \left(\frac{1}{\delta_d}\right)^d = \left(\frac{L\sqrt{d}}{\eps}\right)^d.
		\]
	\end{proof}

	\section{An exponential separation between the size of monotone and arbitrary threshold networks}
	By the universal approximation result for monotone threshold networks, Theorem \ref{thm:universalapprox}, we can approximate monotone functions by monotone networks, arbitrarily well. In this section, we focus on the efficiency of the approximation, in terms of the number of neurons used. Are there functions such that monotone networks approximating them provably require a much larger size than networks that are allowed to have negative parameters? We show that the answer is positive when seeking an $\ell_{\infty}$-approximation smaller than $\eps$ for any $\eps \in [0,1/2).$
	
	Our proof of this fact builds on findings from the monotone complexity theory of Boolean functions. We begin by defining monotone circuits.
	
	\begin{definition}[Monotone Boolean circuits]
		A monotone circuit is a Boolean circuit that only uses AND and OR gates. To allow monotone circuits to represent functions with a range in some bounded interval of $\NN$ (rather than $\{0,1\})$, we will allow the output gate of the circuit to apply a linear function, with weights $1, 2, 2^2, 2^3, \dots$, to its input. In this case, the input bits for the output gate will be the bit representation of the output.
		
		Note that with definition if $C:\{0,1\}^d \to \NN$ is a monotone circuit, then $C(x) \leq C(y)$, whenever $x \leq y$, for $x, y \in \{0,1\}^d.$ 
	\end{definition}
	Specifically, we demonstrate a reduction, which goes in both ways and is compatible with the monotonicity constraint, between Boolean circuits and neural networks. The rest of the section is devoted to proving Theorem \ref{thm:monotoneeffic} and is defined into 3 parts:
		\begin{itemize}
			\item We first explain how to extend a Boolean function from the discrete cube $\{0,1\}^d$, the natural domain for circuits, to the solid cube $[0,1]^d$, where neural networks perform computations. 
			\item We then describe our two reductions. The first reduction constructs a monotone neural network from a monotone circuit, while the second reduction shows how to build a circuit from a general neural network.
			\item Finally we use known results about hard Boolean functions, along with our reductions to establish Theorem \ref{thm:monotoneeffic}. 
		\end{itemize}
	
	\subsection{The harmonic extension}
	Since Boolean circuits represent functions from the discrete cube, $f:\{0,1\}^d \to \NN$, for the reduction we need to have a way to extend such functions to the solid cube $[0,1]^d$. For our proof, we choose the harmonic, or multi-linear extension, defined as follows.
	
	\begin{definition}[Harmonic extension] \label{def:harmonic}
		Let $f:\{0,1\}^d \to \RR$ and $B_t$ a standard Brownian motion on $\RR^d$, and
		define the martingale $X_t$ by $$d[X_t]_i = [B_t]_i{\bf 1}_{[X_t]_i(1-[X_t]_i) > 0}.$$
		Notice that, since $X_t$ is a bounded martingale, it converges to a limit $X_\infty$, and that almost surely,
		$$X_\infty \in \{0,1\}^n.$$
		Define $\hat{f}:[0,1]^d \to \RR$, the harmonic extension of $f$ by,
		\begin{equation} \label{eq:harmext}
			\hat{f}(x) = \EE\left[f(X_\infty)|X_0 = x\right].
		\end{equation}
	\end{definition}
The martingale representation, given in the above definition, is a general solution to the Dirichlet problem and hence produces harmonic extensions in general domains. On the discrete hypercube, these solutions take a more familiar and intuitive form. Indeed, every function on $\{0,1\}^d$ takes the form of a multi-linear polynomial which allows to represent the harmonic extension as a multi-linear extension.

\begin{claim}
	Let $f:\{0,1\}^d \to \RR$ and write $f$ as a multi-linear polynomial $f(x) = \sum\limits_{I\subset [d]} a_Ix^I.$ Then, if $\hat{f}$ is the harmonic extension of $f$, we have,
	\begin{equation} \label{eq:multilinear}
		\hat{f}(x) = \sum\limits_{I\subset [d]} a_Ix^I, 
	\end{equation}
	where now we allow $x \in [0,1]^d$.
\end{claim}
\begin{proof}
	Observe that $\hat{f}$ is a harmonic function. This follows immediately from the fact that $X_t$ in \eqref{eq:harmext} has independent coordinates that solve the heat equation (see \cite[Section 9]{oksendal2003stochastic} for more details). On the other hand, when considered as a function on $\RR^d$, $\sum\limits_{I\subset [d]} a_Ix^I$ is also harmonic since it is a multi-linear polynomial.
	
	By construction, for every $x \in \{0,1\}^d$ we have $\hat{f}(x) = f(x) = \sum\limits_{I\subset [d]} a_Ix^I$. The claim now follows from the uniqueness of the harmonic extension, \cite[Theorem 9.1.1]{oksendal2003stochastic}.
\end{proof}

To supply some further intuition, beyond the equivalent definitions, it is instructive to think about graph properties. In this case, for a graph with vertex set $[n]$, the domain is the adjacency matrix of the graph $\{0,1\}^{n\times n}$ and $f:\{0,1\}^{n\times n} \to \{0,1\}$ checks if some property is satisfied; e.g. is the graph $3$-colorable or does it contain a perfect matching. The harmonic extension has a nice representation in this case. If $x \in [0,1]^{d\times d}$, and we define the inhomogeneous random graph $G(x)$ to have an edge between $i,j \in [n]$ with probability $[x]_{i,j}$, independently from all other edges.
The harmonic extension is then given by,
$$\hat{f}(x) = \PP\left(f(G(x)) =1\right).$$
Thus, if $f$ checks whether a graph contains a perfect matching, $\hat{f}$ encodes the probability that a randomly generated graph will contain a perfect matching.

Having established our method of extension for functions on $\{0,1\}^d$, we now record the following two properties.
	\begin{lemma} \label{lem:harmprop}
		Let $f:\{0,1\}^d \to \RR$ and let $\hat{f}: [0,1]^d\to \RR$ be its harmonic extension. Then,
		\begin{enumerate}
			\item For every $x \in \{0,1\}^d$, $\hat{f}(x) = f(x)$.
			\item If $f$ is monotone, so is $\hat{f}$.
			\item If $m \leq f(x) \leq M$, for some $m < M$ and every $x \in \{0,1\}^d$, then $\hat{f}$ is $(M-m)\sqrt{d}$-Lipschitz. 
		\end{enumerate}
	\end{lemma}
	\begin{proof}
		The first property is immediate from \eqref{eq:harmext}, the defining equation of $\hat{f}$. Indeed if $x \in \{0,1\}^d$ then, by definition of $X_t$, $\PP\left(X_\infty  = x\right) = 1$. We thus focus on the two other properties. For $i \in [d]$ define the discrete directional derivative 
		\begin{align*}
		    \partial_if(x):=& f([x]_1,\dots [x]_{i-1},1,[x]_{i+1},\dots,[x]_d) \\
		    &- f([x]_1,\dots [x]_{i-1},0,[x]_{i+1},\dots,[x]_d).
		\end{align*}
		Both stated properties follow from the following observation, which follows from \eqref{eq:multilinear}, the multi-linear representation of $\hat{f}$,
		$$\frac{d}{dx_i} \hat{f} = \widehat{\partial_i f}.$$
		On the left side of the identity, we have the Euclidean partial derivative of the harmonic extension, while on the right side, we are taking the harmonic extension of the discrete directional derivative.
		
		Equation \eqref{eq:harmext} shows that for every $y \in [0,1]^d$, $\widehat{\partial_i f}(y)$ can be written as a convex combination of $\{\partial_i f(x)\}_{x\in \{0,1\}^d}$. Now, if $f$ is monotone, then for every $x \in \{0,1\}^d$ and every $i \in [d]$, $\partial_i f(x) \geq 0$, which implies that, as a convex combination, $\widehat{\partial_i f}(y) \geq 0$ for every $y \in [0,1]^d$ and $i \in [d]$. The last property is equivalent to the monotonicity of $\hat{f}$, in the sense of \eqref{eq:montone}.
		
		For the last property, to show that $\hat{f}$ is Lipschitz, it will be enough to show that the norm of its (Euclidean) gradient is bounded by $(M-m)\sqrt{d}$ on $[0,1]^d$. If $m \leq f(x) \leq M$ for every $x \in \{0,1\}^d$, then by definition $|\partial_if(x)| \leq M-m$ for every $x \in \{0,1\}^d$. Again, as a convex combination, this implies $|\widehat{\partial_i f}(y)| \leq M-m$ for every $y \in [0,1]^d$. Applying this estimate to the gradient, we see,
		\begin{align*}
		    \|\nabla \hat{f}(y)\| &= \sqrt{\sum\limits_{i=1}^d |\frac{d}{dx_i}\hat{f}(y)|^2} \\
		    &\leq  \sqrt{\sum\limits_{i=1}^d (M-m)^2} =(M-m)\sqrt{d}.
		\end{align*}
	\end{proof}
	\subsection{Reductions between networks and circuits}
	\paragraph{A reduction from monotone circuits to monotone networks}
	Our aim here is to show that if some $f:\{0,1\}^d\to \NN$ cannot be computed by small monotone circuits, then it cannot be approximated by small monotone networks. We show this by demonstrating a reduction from monotone circuits to monotone networks. Given a monotone network that approximates $\hat{f}$ in $[0,1]^d$, we show how to construct a monotone circuit of comparable size to calculate $f$. The proof uses the following fact established in~\cite{beimel2006monotone}:
	\begin{theorem}[{\cite[Theorem 3.1]{beimel2006monotone}}]\label{thm:beimel}
		Let $f$ be a Boolean threshold function with $r$ inputs and non-negative weights $w_1, \ldots w_r$ and bias $T$.
		Namely $f(x_1, \ldots x_r)={\bf 1}(w_1x_1+\ldots w_rx_r\geq T)$. Then there is a monotone De Morgan circuit (that is, a circuit with AND as well as OR gates, but without NOT gates) computing $f$ with $O(r^k)$ gates, where $k$ is a fixed constant independent of the number of inputs $r$. 
	\end{theorem}
	
	It follows from Theorem~\ref{thm:beimel} that if there was a monotone network (with threshold gates) of a given size approximating a harmonic extension $\hat{f}$, we could replace each gate with a polynomially sized monotone De Morgan circuit entailing a polynomial blowup to the size of the network. This construction, in turn, would imply the existence of a monotone De Morgan circuit of comparable size computing $f$ over Boolean inputs. We formalize this idea below:
	\begin{lemma} \label{lem:monoimpossibility}
		Let $f:\{0,1\}^d \to \{0,1,\dots,d\}$, and suppose that any monotone circuit, using only AND and OR gates, which calculates $f$ has size at least $\Omega(\exp(d^\alpha))$, for some $\alpha > 0$. Then, there exists an $\alpha' > 0$, such that if $N$ is a monotone threshold network of size $O\left(\exp(d^{\alpha'})\right)$, there exists ${x} \in [0,1]^{d}$, such that, 
		\[
		|N(x)-\hat{f}(x)| \geq \frac{1}{2}.
		\]
	\end{lemma}
	
	\begin{proof}
		Suppose towards a contradiction that there is a monotone network $N$ of size $s=\exp(d^{\alpha'})$ that approximates $\hat{f}$ in $[0,1]^{d}$
		within error less than $\frac{1}{2}$. Then, restricting $N$ to Boolean inputs, and rounding the output to the nearest integer would yield a monotone threshold circuit $C_N$ that computes $f$ (exactly) on Boolean inputs and has size $s+d$. With more details, every neuron in $N$ corresponds to one monotone threshold gate in $C_N$, and $d$ additional gates to account for the rounding.
		
		By Theorem~\ref{thm:beimel} the monotone circuit complexity of a circuit with only AND and OR gates (De Morgan circuit) computing a threshold function with positive coefficients is polynomial. Therefore, we claim that the existence of $C_N$ entails the existence of a monotone circuit $C'_N$ with AND and OR gates of size\footnote{By size we mean the number of gates in the circuit.} $O(s^t)$, for some constant $t \geq 1$, that computes $f$. 
		
		Indeed, by Theorem~\ref{thm:beimel} we may construct $C'_N$ by replacing every monotone threshold gate, $g$, in $C_N$ by a monotone De Morgan circuit $C_g$ computing $g$. As the size of $C_g$ is polynomial in the number of inputs to $g$ (and hence upper bounded by $s^k$ for an appropriate constant $k$) these replacements result in a polynomial blowup with respect to the size of $C_N$: the size of $C'_N$ is at most $(s+1)s^k=O(s^{k+1})$. Therefore setting $t$ to be $k+1$ we have that the size of $C'_n$ is at most $O(s^t)=O\left(e^{td^{\alpha'}}\right)$. If $\alpha' > 0$ is such that $(k+1)d^{\alpha'} \leq d^{\alpha}$, this contradicts the assumption about $f$.
	\end{proof}
	
	\paragraph{A reduction from general networks to circuits}
	Now we show that given a Boolean circuit to compute a given function $f: \{0,1\}^d \to \NN$, we can construct a general (not necessarily monotone) threshold network of comparable size to approximate $\hat{f}$ arbitrarily well. Formally we shall prove.
	\begin{lemma} \label{lem:nettocirc}
			Let $f:\{0,1\}^d\to \{0,1,\dots,d\}$ and let $\delta \in (0,1)$  and suppose that there exists a Boolean circuit of polynomial size which computes $f$. For every fixed $\eps>0$ there exist a general threshold network $N$ of polynomial size in $d$, such that for all ${x} \in [0,1]^{d}$, 
		\[
		|\hat{f}({x}) - N({x})| \leq \eps.
		\]
	\end{lemma}
	To estimate $\hat{f}(x)$ we can realize independent (polynomially many) copies
	of $X_\infty|X_0 = x$ and estimate 
	$$\EE\left[f(X_\infty)|X_0 = x\right],$$
	as in \eqref{eq:harmext}. To implement this idea, two issues need to be addressed:
	\begin{itemize}
		\item The use of randomness by the algorithm (our neural networks do not use randomness).
		\item The algorithm dealing with vectors in $[0,1]^{d}$ that may need infinitely many bits to be represented in binary expansion. 
	\end{itemize}
	
	We first present a randomized polynomial-time algorithm, denoted $A$, for approximating $\hat{f}$. We then show how to implement it with a (deterministic) threshold network. Algorithm $A$ works as follows. Let $q(),r()$ be polynomials to be defined later. First, the algorithm (with input $x$) only considers the $q(d)$ most significant bits in the binary representation of every coordinate in $x$. Next, it realizes $r(d)$ independent copies of $\{X^i\}_{i=1}^{r(d)} \sim X_\infty|X_0 = \tilde{x}$, and computes, using a Boolean circuit, for each one of these copies $f(X^i)$. The algorithm outputs $\sum\limits_{i = 1}^{r(d)}\frac{f(X^i)}{r(d)}$. Clearly, the running time of this algorithm is polynomial in the size of the Boolean circuit. 
	
	Let ${\widetilde{x}}$ be the vector obtained from $x$ when considering the $q(d)$ most significant bits in each coordinate, and observe 
	\begin{equation} \label{eq:eqinlaw}
		A(x) \stackrel{\rm law}{=} A({ \widetilde{x}}).
	\end{equation} 
	Keeping this in mind, we shall first require the following technical result.
	\begin{lemma}\label{lem:Chernoff}
		Let $f:\{0,1\}^d\to \{0,1,\dots,M\}$ and let $\delta \in (0,1)$ be an accuracy parameter. Then, if ${\tilde{x}} \in [0,1]^{d}$ is such that every coordinate ${ [\tilde{x}]_{i}}$ can be represented by at most $q(d)$ bits,
		\begin{equation} \label{eq:truncatedapprox}
			\mathbb{P}\left(|\hat{f}({{\tilde{x}}})-A({\tilde{x}})|>\delta\right) \leq 2e^{-\frac{2r(d)\delta^2}{M^2}}.
		\end{equation}
		As a consequence, for every $x \in [0,1]^{d}$,
		\[
		\mathbb{P}\left(|\hat{f}(x)-A(x)|>\delta + \frac{Md}{\sqrt{2}^{q(d)}}\right) \leq 2e^{-\frac{2r(d)\delta^2}{M^2}}.
		\]
	\end{lemma}
	\begin{proof}
		We first establish the first part.
		Indeed, $A({\tilde{x}}) = \sum\limits_{i=1}^{r(d)}\frac{f(X^i)}{r(d)}$ where ${\frac{f(X_i)}{r(d)}}$ are independent and by assumption on $f$, bounded in $[0,\frac{M}{r(d)}]$. Moreover, we have from the definition of the harmonic extension 
		$$\mathbb{E}\left[A({\tilde{x}})\right] = \EE\left[f(X_\infty)|X_0 =  \tilde{x}\right] = \hat{f}({\tilde{x}}).$$
		 Thus, \eqref{eq:truncatedapprox} is a consequence of Hoeffding's inequality inequality, \cite[Theorem 2.8]{boucheron2013concentration}. Now, ${\tilde{x}}$ was obtained from ${\bf x}$ by keeping the $q(n)$ most significant bits. Thus, $\|{ x} - { \tilde{x}}\| \leq \sqrt{\frac{d}{2^{q(n)}}}$, and by the third property in Lemma \ref{lem:harmprop},
		\[
		|\hat{f}({x}) - \hat{f}({\tilde{x}})| \leq \frac{Md}{\sqrt{2}^{q(d)}}.
		\]
		So, because of \eqref{eq:eqinlaw}, and \eqref{eq:truncatedapprox},
		\begin{align*}
			&\mathbb{P}\Bigg(|\hat{f}({{x}})-A({x})|>\delta + \frac{Md}{\sqrt{2}^{q(d)}}\Bigg) \\
			&= \mathbb{P}\left(|\hat{f}({{x}})-A({ \tilde{x}})|>\delta + \frac{Md}{\sqrt{2}^{q(d)}}\right)\\
			&\leq \mathbb{P}\left(|\hat{f}({\bf x}) - \hat{f}({\tilde{x}})| + |\hat{f}({{\tilde{x}}})-A({\tilde{x}})|>\delta + \frac{Md}{\sqrt{2}^{q(d)}}\right)\\
			&\leq \mathbb{P}\left(|\hat{f}({{\tilde{x}}})-A({\tilde{x}})|>\delta\right) \leq 2e^{-\frac{2r(d)\delta^2}{M^2}}.
		\end{align*}
	\end{proof}
	
	We next show how to implement this algorithm by a neural network of polynomial size that does not use randomness.
	\begin{lemma}\label{lem:construction}
		Let $f:\{0,1\}^d\to \{0,1,\dots,M\}$ and let $\delta \in (0,1)$ be a fixed constant. Suppose that there exists a Boolean circuit of polynomial size that computes $f$. Then, there exists a neural network of polynomial size $N$ such that for every ${x} \in [0,1]^{d},$
		\[
		|\hat{f}({x})-N({x})| \leq \delta + \frac{Md}{\sqrt{2}^{q(n)}}.
		\]
	\end{lemma}
	\begin{proof}
		By standard results regarding universality of Boolean circuits (See~\cite[Theorem 20.3]{shalev2014understanding} or \cite[Theorem 9.30]{sipser1996introduction}) algorithm $A$ can be implemented by a neural network
		$N'$ of polynomial size that get as additional input $dq(d)r(d)$ random bits. In the previous claim, we have also used the fact that $A$ implements a Boolean circuit of polynomial size to compute $f$. 
		
		In light of this, our task is to show how to get rid of these random bits.
		The number of points in $[0,1]^{d}$ represented by at most $q(d)$ bits in each coordinate is at most $2^{dq(n)}$. On the other hand by the first part of 
		Lemma~\ref{lem:Chernoff} the probability for a given input ${\tilde{x}}$ (with each coordinate in ${\tilde{x}}$ represented by at most $q(d)$ bits) that $|N'({ \tilde{x}})-\hat{f}({{\tilde{x})}}| \geq \delta$ is $2^{-\Omega(r(d)/M^2)}$. As long as $r(d)$ is a polynomial with a large enough degree, with respect to $q(d)$ and $M$, we have that $2^{dq(d)-\Omega(r(n)/M^2)}<1$. Hence, there must exist a fixed choice for the $dq(d)r(d)$
		(random) bits used by the algorithm such that for \emph{every} input ${x}$ the additive error of the neural network $N'$ on $x$
		is no larger than $\delta$. Fixing these bits and hard-wiring them to $N'$ results in the desired neural network $N.$
		The final result follows from the proof of the second part of Lemma \ref{lem:Chernoff}.
	\end{proof}
Finally we may prove Lemma \ref{lem:nettocirc}

\begin{proof}{proof of Lemma \ref{lem:nettocirc}}
	Set $\delta = \frac{\eps}{2}$ and let $N$ be the network constructed in Lemma~\ref{lem:construction} with $M = d$ and accuracy parameter $\delta$. Choose $q(d)$ which satisfies $q(d)>\log (\frac{4d^{4}}{\eps^2} )$. Thus, Lemma \ref{lem:construction} implies, for every ${x} \in [0,1]^{d}$:
	\[
	|\hat{f}(x)-N({x}))| \leq \delta + \frac{d^2}{\sqrt{2}^{q(d)}} \leq \frac{\eps}{2} + \frac{\eps}{2} = \eps.
	\]
	The claim is proven.
\end{proof}
	\subsection{Proving exponential separation}
	To apply our reductions, Lemma \ref{lem:monoimpossibility} and Lemma \ref{lem:nettocirc}, and complete the proof of Theorem \ref{thm:monotoneeffic} we now require a function $f:\{0,1\}^d \to \{0,1,\dots,d\}$ that is
	\begin{enumerate}
		\item Computable by a polynomial-size Boolean circuit.
		\item Cannot be computed by any Monotone circuit of size $e^{d^{\alpha}}$ for some $\alpha >0$.
	\end{enumerate}
	 
	 The Tardos function, introduced in \cite{tardos1988gap}, will satisfy these conditions. The Tardos function builds upon the seminal work of Razborov, \cite{razborov1985lower}, who studied the hardness of monotone computation for perfect matching. The function is constructed as a graph invariant and is always sandwiched between the clique and chromatic number of the graph. Below we outline the necessary properties.
	\begin{theorem}[{\cite{tardos1988gap}}]\label{thm:Tardos}
		Assume $d$ is a square. There exists a monotone function $f_{\mathrm{Tar}}:\{0,1\}^d \to \{0,1,\dots \sqrt{d}\}$ which can be computed in polynomial time.
		Furthermore, any Boolean circuit with only AND and OR gates (monotone circuit) that computes $f_{\mathrm{Tar}}$ has size $e^{\Omega( d^{\alpha})},$ for some $\alpha > 0$. 
	\end{theorem}
	Razborov's original work \cite{razborov1985lower} proves a similar result for finding a perfect matching. However, the bound is super-polynomial instead of truly exponential. In what comes next, we could have also used the indicator function of a perfect matching in a graph to obtain another separation, albeit weaker, result.

	We now let $\widehat{f_{\mathrm{Tar}}}$ be the harmonic extension of the Tardos function. Theorem \ref{thm:monotoneeffic} will follow as an immediate consequence of the following more specific theorem.
	\begin{theorem}\label{thm:monotone}
		The function $\widehat{f_{\mathrm{Tar}}}:[0,1]^d\to[0,\sqrt{d}]$ is a monotone function with, which satisfies the following: 
		\begin{itemize}
			\item There exists $\alpha' > 0$ such that if $N$ is a monotone threshold network of size less than $e^{d^{\alpha'}}$, there exists ${x} \in [0,1]^{d}$, such that
			\[
			|N(x)-\widehat{f_{\mathrm{Tar}}}(x)| \geq \frac{1}{2}.
			\]
			\item For every fixed $\eps>0$ there exist a general threshold network $N$ of polynomial size in $n$, such that for all ${x} \in [0,1]^{d}$, 
			\[
				|N(x)-\widehat{f_{\mathrm{Tar}}}({x})| \leq \eps.
			\]
				\end{itemize}
	\end{theorem}
	\begin{proof}
		First observe that by Theorem \ref{thm:Tardos}, $f_{\mathrm{Tar}}$ is monotone, which implies the same thing to $\widehat{f_{\mathrm{Tar}}}$ by Lemma \ref{lem:harmprop}. keeping Theorem \ref{thm:Tardos} in mind, the first property is an immediate corollary of Lemma \ref{lem:monoimpossibility}. 
		
		For the second property, we use standard results to Turing machines into threshold circuits. Specifically, a Turing machine that decides an algorithmic problem with inputs of size $d$ in time $t(d)$ can be realized as a threshold circuit with $O(t(d)^2)$ gates, solving the algorithmic problem on every input of length $d$,~\cite{sipser1996introduction,shalev2014understanding}.
		Thus, since $f_{\mathrm{Tar}}$ is computable in polynomial time, there exists a Boolean circuit of polynomial size that computes $f$. The second property is now a consequence of Lemma \ref{lem:nettocirc}.
	\end{proof}

            \section{Experiments}

            In this section, we report on the numerical experiments we've conducted to support our theoretical results. The construction, presented in Section \ref{sec:construction}, was implemented in Matlab\footnote{The code is publicly available at \url{https://github.com/danmiku/MonotoneNetworks}.}.

         \paragraph{Interpolation error} We evaluated the interpolation error of our construction in different dimensions with randomly generated monotone data sets. The evaluation was performed in different dimensional scales and we constructed interpolating networks with input dimensions $10$, $50$, and $200$. In each dimension $d$, we evaluated three data sets of sizes $d, 2d, 5d$. In each case, the data set of size $n$ was constructed in the following form: We first generated $n$ points $(x_i)_{i=1}^n$ uniformly at random, and independently from one another, from the unit cube $[0,1]^d$. The labels were then generated by applying some monotone function $f:\RR^d \to \RR$ on each point, obtaining $\{f(x_i)\}_{i=1}^n$. To evaluate the resulting network, $N$, in each case we computed the interpolation error as $\sum\limits_{i=1}^n \left(f(x_i)-N(x_i)\right)^2$. The monotone functions we considered in our experiments were $f(x) = \|x\|_2$, and  $f(x) = e^{\|x\|_2}$. Unsurprisingly, and in line with Theorem \ref{thm:construction}, the interpolation error was $0$ in each considered case. We repeated the same experiments where the points $\{x_i\}_{i=1}^n$ formed a grid $\frac{1}{\sqrt[d]{n}}\mathbb{Z}^d \cap [0,1]^d$ and obtained the same results.

         \paragraph{Extrapolation error} To evaluate the extrapolation error of our construction we followed similar steps to the one detailed in Section \ref{sec:universal}. More specifically, we constructed monotone data sets with $\{x_i\}_{i=1}^n$ points forming a grid $\frac{1}{\sqrt[d]{n}}\mathbb{Z}^d \cap [0,1]^d$ and labels taken as $\{f(x_i)\}_{i=1}^n$. The extrapolation error of the obtained network $n$ was evaluated by estimating the $L_\infty$ error $E:=\max\limits_{x \in [0,1]^d} |f(x) - N(x)|$.
         
        To evaluate the error $E$, instead of directly computing it, we employed a straightforward stochastic approximation scheme. Thus, we sampled a new batch of more than $6^d$ random points $\{z_i\}_{i=1}^{6^d}$, taken from $[0,1]^d$ and augmented the sample by the point $z_0 := (1-10^{-7},\dots, 1-10^{-7})$, at which we expect to approximately find the maximal error. The estimated error was then computed as $\hat{E}:=\max\limits_{z_i}|f(z_i) - N(z_i)|$. The error is reported, rounded to 4 digits, in Table \ref{tab:extra} below.

    \begin{table}[!htb] \footnotesize
    	\begin{minipage}{.5\linewidth}
    		 
    		 \centering
    \begin{tabular}{|c|c|c|c|}
        \hline
        $f(x) = \|x\|_2$ & \textbf{$n = 2^d$} & \textbf{$n = 3^d$} & \textbf{$n = 4^d$} \\
        \hline
        \multicolumn{1}{|c|}{\textbf{$d = 2$}} & \makecell{1.4037 \\ \bf (1.4142)} & \makecell{0.6865 \\ \bf (0.7071)} & \makecell{0.4499\\ \bf(0.4714)} \\
        \cline{1-1}
        \hdashline
        \textbf{$d = 3$} & \makecell{1.6858\\ \bf (1.732)}& \makecell{0.8248\\ \bf (0.8660)}& \makecell{0.5611 \\ \bf (0.5773)} \\
        \cline{1-1}
        \hdashline
        \textbf{$d = 4$} & \makecell{1.9151 \\ \bf (2.0)} & \makecell{0.9625 \\ \bf(1.0)} & \makecell{0.6314\\\bf(0.6666)} \\
        \cline{1-1}
        \hdashline
        \textbf{$d = 5$} & \makecell{2.1732 \\ \bf(2.2360)} & \makecell{0.9834\\\bf(1.1180)} & \makecell{0.6688\\\bf(0.7453)}  \\
        \cline{1-1}
        \hdashline
        \textbf{$d = 6$} & \makecell{2.2610\\\bf(2.4494)} & \makecell{1.1024\\\bf (1.2247)} & \makecell{0.6950 \\ \bf (0.8164)} \\
        \cline{1-1}
        \hline
    \end{tabular}
\end{minipage}
\begin{minipage}{.5\linewidth}
	    
	    \centering
    \begin{tabular}{|c|c|c|c|}
        \hline
        $f(x) = e^{\|x\|_2}$ & \textbf{$n = 2^d$} & \textbf{$n = 3^d$} & \textbf{$n = 4^d$} \\
        \hline
        \multicolumn{1}{|c|}{\textbf{$d = 2$}} & \makecell{3.1133 \\ \bf (5.8170)} & \makecell{2.0851 \\ \bf (2.9085)} & \makecell{1.5461\\ \bf(1.939)} \\
        \cline{1-1}
        \hdashline
        \textbf{$d = 3$} & \makecell{4.6522\\ \bf (9.7899)}& \makecell{3.2748\\ \bf (4.8949)}& \makecell{2.4792 \\ \bf (3.2633)} \\
        \cline{1-1}
        \hdashline
        \textbf{$d = 4$} & \makecell{6.3891 \\ \bf (14.7781)} & \makecell{4.6708 \\ \bf(7.3890)} & \makecell{3.5954\\\bf(4.92603)} \\
        \cline{1-1}
        \hdashline
        \textbf{$d = 5$} & \makecell{8.3565 \\ \bf(20.9217)} & \makecell{6.2976\\\bf(10.4608)} & \makecell{4.9162\\\bf(6.9739)}  \\
        \cline{1-1}
        \hdashline
        \textbf{$d = 6$} & \makecell{10.5824\\\bf(28.3710)} & \makecell{8.1791\\\bf (14.1855)} & \makecell{6.4633 \\ \bf (9.4570)} \\
        \cline{1-1}
        \hline
    \end{tabular}
\end{minipage}
\caption{\small Extrapolation errors for the norm function $f(x) = \|x\|_2$ on the left, and the exponential norm function $f(x) = e^{\|x\|_2}$ on the right.} \label{tab:extra}
\end{table}

 The errors were evaluated in dimensions $d = 2,3,4,5,6$., and the columns correspond to the $3$ different sizes of the grid of interpolated points $\mathbb{Z}^d\cap [0,1]^d$, $\frac{1}{2}\mathbb{Z}^d\cap [0,1]^d$, and $\frac{1}{3}\mathbb{Z}^d\cap [0,1]^d$. Extrapolation error bounds predicted by Theorem \ref{thm:universalapprox} appear in brackets. For a grid of size $m^d$ the error bound was calculated, as in Theorem \ref{thm:universalapprox}, by $\frac{\sqrt{d}L}{m-1}$, where $L$ is the Lipschitz constant of $f$ and $\frac{\sqrt{d}}{m-1}$ is the diameter of each cell in the grid $\frac{1}{m}\mathbb{Z}^d$. The Lipschitz constant of $\|x\|_2$ is $1$ and on $[0,1]^d$ the Lipschitz constant of $e^{\|x\|_2}$ is $e^{\sqrt{d}}$. To further demonstrate Theorem \ref{thm:universalapprox}, in the case $d = 2$, for $f(x) = e^{\|x\|_2}$ we have also evaluated the extrapolation error when the network interpolates the grid $\frac{1}{101}\mathbb{Z}^2 \cap[0,1]$ and obtained an error of $0.0578$  with predicted error $0.0581$.

\paragraph{Visualizations} Finally, we provide visualizations of our obtained networks when $d = 2$. Again, we considered the two functions $f(x) = \|x\|_2$ and $f(x) = e^{\|x\|_2}$. All networks were constructed by interpolating the functions' values over grids of different sizes. Specifically, we considered three different scales, $10 \times 10$, $20 \times 20$, and $30 \times 30$ grids. The plots for $f(x)= \|x\|_2$ appear in Figure \ref{fig:norm} on the left, and the plots for $f(x)= e^{\|x\|_2}$ appear in Figure \ref{fig:norm}, on the right. In addition, each figure contains a plot of the original function for reference. 
In each figure, the top left corner contains a plot of interpolation over a $10 \times 10$ grid, the top right corner is over a $20 \times 20$ grid, and the bottom left corner is over a $30 \times 30$ grid. The bottom right corner plots the target function.
As can be seen, when the grid becomes finer our interpolating network becomes more similar to the target function.

\begin{figure}
	\begin{minipage}{.5\textwidth}
		\centering
		\includegraphics[width=0.9\linewidth]{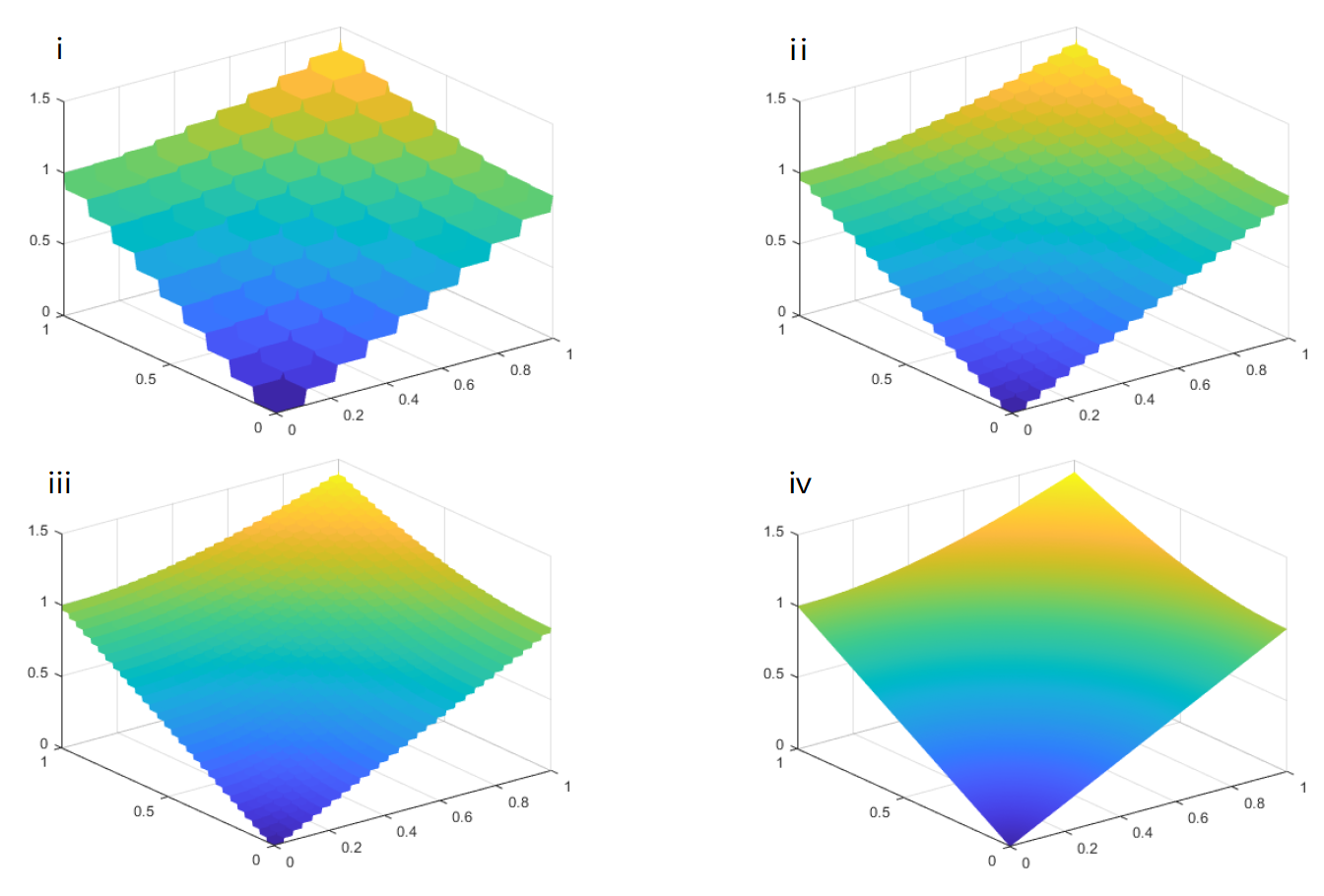}
	\end{minipage}
	\begin{minipage}{.5\textwidth}
		\centering
		\includegraphics[width=0.9\linewidth]{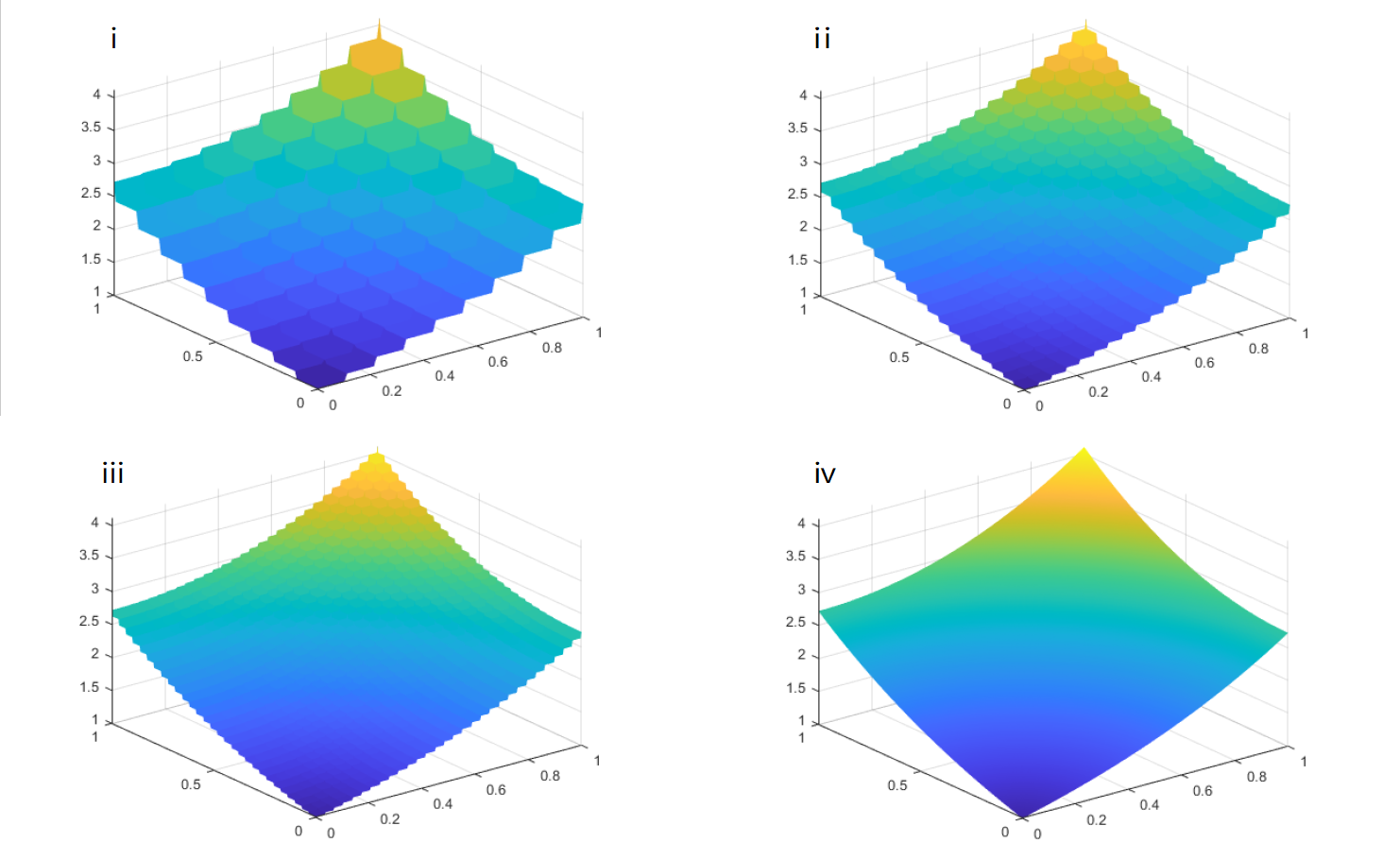}
	\end{minipage}
		\captionof{figure}{Plots of the monotone interpolating networks with labels taken from the function $f(x) = \|x\|_2$ on the left, and $f(x) = e^{\|x\|_2}$ on the right. The plots contain interpolations over three grids: (i) $10 \times 10$ grid, (ii) $20\times20$ grid, (iii) $30\times 30$ grid, and (iv) contains a plot of the original function.}\label{fig:norm}
\end{figure}

\section{Conclusion}

	We studied neural networks with non-negative weights and examined their power and limitations in approximating monotone functions.  
	
	Our results reveal that for the ReLU activation, restricting the weights to be non-negative severally limits the ability of the model to express monotone functions. For threshold activation, we have shown that the restriction to positive parameters is less severe and that universality can be achieved at constant depth. 
	In addition, we have shown that monotone neural networks can be much more resource-consuming, in terms of the number of neurons needed to approximate monotone functions. 
	
	We focused on the threshold activation function. It is an interesting direction to extend our results for other activation functions such as sigmoids. For the universality result of depth 4 monotone networks it seems plausible that one could approximate thresholds by sigmoids and prove that monotone networks of depth 4 with sigmoids are universal approximators of monotone functions. For our lower bounds, based on the matching function $m$ it appears that new ideas are needed to show a super polynomial separation between the size needed for monotone as opposed to arbitrary networks with sigmoids to approximate $m$.
	
 Recent results in Boolean circuit complexity have demonstrated the existence of monotone functions that a polynomial-sized circuit of bounded depth can compute but require super polynomial monotone circuits to compute~\cite{chen2017addition,garg2018monotone}. 
 It is reasonable to expect that one could use our machinery to demonstrate monotone real functions that can be computed by a bounded depth circuit yet require super polynomial size to be approximated by monotone networks. Proving separation results between monotone and non-monotone networks with respect to the square loss is another avenue for further research. 
	
	One aspect we did not consider here is learning neural networks with positive parameters using gradient descent. It would be interesting to examine the efficacy of gradient methods both empirically and theoretically. Such a study could lead to further insights regarding methods that ensure that a neural network approximating a monotone function is indeed monotone. Finally, we did not deal with the generalization properties of monotone networks: Devising tight generalization bounds for monotone networks is left for future study. 
	\section*{Acknowledgements}
	We thank  Arkadev Chattopadhyay for helpful feedback and Todd Millstein for discussing~\cite{sivaraman2020counterexample} with us which led us to think about monotone neural networks. We are grateful to David Kim for implementing our construction of a monotone neural network and testing it over several monotone data sets. Finally, we thank Bruno Pasqualotto Cavalar for bringing to our attention the work done in~\cite{chen2017addition,garg2018monotone}.

	\bibliographystyle{plain}
	\bibliography{reference}

\end{document}